\documentclass{article} 

\usepackage{times}
\usepackage{graphicx} 
\usepackage{subfigure} 
\usepackage{natbib}
\usepackage{algorithm}
\usepackage{algorithmic}
\usepackage{hyperref}
\usepackage[accepted]{icml2015}
\usepackage{enumitem} 

\usepackage[framemethod=TikZ]{mdframed}

\newcommand{\SSp}{\mathcal{S}}
\newcommand{\ASp}{\mathcal{A}}
\newcommand{\CSp}{\mathcal{C}}

\usepackage{url}
\usepackage{amsthm}
\usepackage{amsmath}
\usepackage{amssymb}
\usepackage{amsfonts}
\usepackage{color}
\newtheorem{Definition}{Definition}
\newtheorem{Lemma}{Lemma}
\newtheorem{Theorem}{Theorem}
\newtheorem{Corollary}{Corollary}

\newtheorem{Assumption}{Assumption}

\begin{document}

\icmltitlerunning{Contextual Markov Decision Processes}

\twocolumn[
\icmltitle{Contextual Markov Decision Processes}

\icmlauthor{Assaf Hallak}{ifogph@gmail.com}
\icmladdress{The Technion,
            Haifa, Israel}
            
\icmlauthor{Dotan Di Castro}{dotan.dicastro@gmail.com }
\icmladdress{Yahoo Labs,
            Haifa, Israel}

\icmlauthor{Shie Mannor}{shie@ee.technion.ac.il}
\icmladdress{The Technion,
            Haifa, Israel}
            
            ]

\begin{abstract}
We consider a planning problem where the dynamics and rewards of the environment  depend on a hidden static parameter referred to as the \emph{context}. The objective is to learn a strategy that maximizes the accumulated reward across all contexts. The new model, called \emph{Contextual Markov Decision Process} (CMDP), 
can model a customer's behavior when interacting with a website (the learner). The customer's behavior depends on gender, age, location, device, etc. Based on that behavior, the website objective is to \emph{determine} customer characteristics, and to \emph{optimize} the interaction between them.
Our work focuses on one basic scenario--finite horizon with a small known number of possible contexts. We suggest a family of algorithms with provable guarantees that learn the underlying models and the latent contexts, and optimize the CMDPs. Bounds are obtained for specific naive implementations, and extensions of the framework are discussed, laying the ground for future research.
\end{abstract}

\section{Introduction}

Markov Decision Processes (MDPs) are commonly used to describe dynamic behavior in multiple fields such as signal processing, robotics, games, advertising, health and queues management \cite{puterman2005markov, white1993survey}. When multiple trajectories are observed from a single source, a question in this context is the following: ``Does each observed trajectory follow the same transition probabilities"? When the answer is affirmative, these transitions can be evaluated through standard maximum likelihood estimation \cite{boas2006mathematical}, and many techniques exist for different setups, most notably are the \emph{Hidden Markov Models (HMMs)} method \cite{elliott1995hidden} in modeling and the \emph{Partially Observed Markov Decision Processes (POMDPs)} \cite{aberdeen2003revised} in control.

However, in many applications there are additional exogenous variables that affect the model. We refer to these variables collectively as the \emph{context}. For example, the temporal behavior of sugar levels for diabetes patients is largely influenced by their age and gender. Similarly, humidity measurements are greatly affected by the geographical location of the measurement device. Since these context variables do not change within each measurement, the standard solution of incorporating them into the state creating a much larger MDP or POMDP seems faulty as it reduces the generalizing power of the model. Specifically, incorporating static features into the state forms distinct unconnected dynamic chains. As transition probability between states with different contexts is always zero, a more compact model would be separate transition matrices for each context instead of one double sized matrix.

\subsection{Motivation for Contextual Dynamics}

A real world example for latent context learning is the problem of \emph{identifying the user}. Consider a large content website. Such a website has two main activities: (a) suggesting relevant content to its users and (b) presenting alluring ads for profit. Current methodologies that determine the relevance of the content and the ads require the user profile: age, gender, income level, device, location, etc. Usually, in order to determine whether a certain user is revisiting the website, mechanisms such as (HTTP) cookies are used. But in many cases these mechanisms are insufficient. What if the website does not have any prior information about the user (also known as the \emph{cold start problem}; \citealt{kohrs2001improving})? Can we learn the user's age or gender by observing his interaction with the website? In other words, given a trajectory of the pages visited by the user can we predict (\emph{cluster} or \emph{classify}) the user's profile? And more importantly, can we take advantage of such clustering and tailor the policy to the user?

This type of problem exists also in scenarios where we have information about the owner of a device, but several users use it and we want to identify them (such as children using their parents tablets). In this work we suggest to model the user interaction in a \emph{Markovian} fashion in order to identify the user \cite{meyn2009markov}.

A more elaborate scenario is when the user has been identified, and we want to optimize the content and ads presented to him where the optimization criterion, for instance, is maximizing the user's time spent in the website. In such cases, we model the interaction of a user as a \emph{Markov Decision Process} where different user's groups may be modeled and optimized according to their context. In on-line advertising, solutions to such optimization problem are highly valuable, where the correct identification of users leads to higher \emph{click through rates} (CTRs; \citealt{richardson2007predicting}). Hence, the ultimate goal is \emph{on-line learning an optimal control when both the context, and the model's parameters are unknown}. Notice that a sub-goal in this case is the one described above: learning the underlying Markov dynamics.

Our work's main contribution is presenting a general algorithm with provable guarantees for the finite horizon episodic contextual MDP setup. Considering a specific implementation, we provide regret analysis and empirical parametric sensitivity analysis. Additionally, we discuss two applicative extensions of the model: the case of infinitely many contexts, and the concurrent Reinforcement Learning (RL) \cite{silver2013concurrent} setting. The reader should bear in mind the solutions suggested are preliminary and our focus is on presenting the problems along with their derived trade-offs, as well as setting the bar for future research.

\subsection{Related Literature}

Many previous works are related to the setup presented in this paper. In Hidden Markov Models (HMMs; e.g., \citealt{elliott1995hidden}) the Markovian state dynamics are latent, and the observed samples are transformation of the sources output. Works by \citealt{wilson1999parametric} and \citealt{radenen2014handling} had considered adding context to HMMs, however the context in their model only affects the observations distribution and not the state dynamics.

A natural extension of HMMs to a control setting is POMDPs \cite{aberdeen2003revised}. CMDPs can be modeled using POMDPs by setting the context variable to be the origin and each possible MDP as a distinct outgoing chain. However, POMDPs are too general and complex to capture the essence of the CMDP setup. In addition, POMDPs usually assume an underlying distribution on the contexts which we refrain from doing. 

The notion of context was borrowed from closely related works in the Multi-Armed Bandits (MAB) literature \cite{sutton1998introduction,bubeck2012regret}, called \emph{Contextual-MAB} \cite{langford2007epoch,lai1985asymptotically}. The extension to the regular setting of MAB is that before the learner plays his turn, a context is presented to the user. Another similar paper is by \citet{maillard2014latent}, describing a setup in which only the rewards depend on an unobserved latent variable. They consider three cases: the reward function and context are known, the reward function is known but the context is not, and where both are unknown. 

Other related literature considers model selection in MDPs. \citet{doya2002multiple} propose an architecture for \emph{multiple model-based reinforcement learning} (MMRL). Their approach decomposes a complex task into multiple domains in time and space and use a responsibility signal to weigh the outputs of multiple models and to gate the learning of the prediction models and controllers. Hence, responsibility signal measure how to mix different models such that various areas in the state space could be more easily modeled. A similar approach was used in learning meta-parameters of motor skills \cite{kober2012reinforcement}. In our work, we try to identify a single source that fits all the space. 

Another relevant problem is that of on-line representation learning \cite{nguyen2013competing,maillard2011selecting}, dealing with finding the best state space while interacting with the environment. Differently, in the CMDP setup all models share the same state space. Finally, \citet{kiseleva2013predicting} consider contextual MDPs, but from a purely applicative perspective--they relate directly to web advertising by modeling user types.

We conclude with a short comparison of CMDPs with other known extensions of the MDP model:
\begin{enumerate}
	\item In \textbf{Contextual HMMs} the context affects only the observation distribution, and there is no control.
	\item \textbf{POMDPs} are a more complex structure generalizing CMDPs. Since in our case the hidden parameters are constant over time a simpler solution might exist. In addition, some distribution over contexts is assumed.
	\item In \textbf{Multi-model RL} the dynamics and rewards are composed of a convex combination of several models, meaning that in each trajectory there can be more than one valid model.	
	\item The problem of \textbf{state representation} is that of finding a suitable state space for given observations. In our case the state space is the same for all models, allowing more efficient solutions.
	\item Models described by \textbf{robust MDPs} \cite{nilim2005robust,wiesemann2013robust} consider uncertainty in the transitions and rewards. CMDPs can be viewed as such, where the uncertainty is not rectangular (around each state-action pair) but singular--determining one transition sets all of them.	
\end{enumerate}

\subsection{Paper Structure}

The paper is organized as follows: In Section \ref{sec:CMDP} we formally define the CMDP setting, compare it to other models from the literature and introduce the general setup. Section \ref{sec:ProblemDefinition} describes the problem in more details and presents a general form algorithm to solve it. One specific instance of the algorithm is analyzed, and eventually some possible extensions are presented. In Section \ref{sec:experiments}, we provide experiments and discuss trade-offs in our setup. Finally, in Section \ref{sec:Conclusions} we conclude and lay down future directions.

\section{Contextual Markov Decision Processes}\label{sec:CMDP}

We begin with defining a standard Markov Decision Process (MDP; \citealt{puterman2005markov}).

\begin{Definition}{(MDP Setup)}\label{def:MDP}
A \textbf{Markov Decision Process} is a tuple $(\SSp, \ASp, p(y|x,a), r(x), \pi_0)$ where $\SSp$ is the state space, $\ASp$ is the action space, $p(y|x,a)$ is the transition probability ($y,x\in \SSp, a\in \ASp$), $r(x)$ is a reward function, and $\pi_0$ is the initial state distribution.
\end{Definition}

Given a deterministic horizon $T$, the learner-interaction is as follows. At the beginning of each episode, an initial state $x_0$ is chosen according to the state distribution $\pi_0$. Afterwards, for each $0\le t \le T$, the learner chooses an action according to a policy $\mu(a_t|x_t)$ where $a_t\in \ASp,x_t\in \SSp$. We note that the policy may be a random function. The environment provides a reward $r(x_t)$ and the next state is (randomly) chosen according to $p(x_{t+1}|x_t,a_t)$. In general, the learner's goal is to maximize the following value function:
\begin{equation*}
J^\mu = \mathrm{E}\! \left[ \sum_{t=0}^{T} r(x_t) \Bigg | x_0 \sim \pi_0 ,\mu(a|x)\right]
\end{equation*}
where the expectation is taken over trajectories with respect to the policy $\mu(a|x)$ and the initial distribution $\pi_0$.

When the MDP parameters are given, the problem of finding the policy which maximizes cumulative reward is known in the literature as \emph{planning} \cite{puterman2005markov,bertsekas1995neuro}. When the MDP parameters are unknown in advance, finding the best policy is known as \emph{Adaptive Control} or \emph{Reinforcement Learning} (RL; \citealt{puterman2005markov,bertsekas1995neuro}).

The following definition establishes the extended model considered in the paper, denoted by \emph{Contextual MDPs}. 

\begin{Definition}\label{def:CMDP}
	Contextual Markov Decision Process (CMDP) is a tuple $(\CSp, \SSp, \ASp, \mathcal{M}(c))$ where $\CSp$ is called the context space, $\SSp$ and $\ASp$ are the state and action space correspondingly, and $\mathcal{M}$ is function mapping any context $c\in \CSp$ to an MDP $\mathcal{M}(c)=(\SSp, \ASp, p^c(y|x,a), r^c(x), \pi^c_0)$.
\end{Definition}
So essentially, CMDP is simply a set of models sharing the same state and action space.

The simplest scenario in a CMDP setting is when the context is simply \emph{observable}. In this setting, the problem reduces to correctly generalizing the model from the context. If the observable context $c$ is \emph{finite} where $|\CSp|=K$ , then with no further assumption, one can simply learn $K$ different models. 

An interesting problem arises when $K$ scales with the number of sampled trajectories. For instance, consider the problem of targeted advertising: Given behavioral patterns and side information of many customers, companies usually seek to group the consumers so they can target their needs and habits. Since side information usually resides in a very large set (for example, the cross-product of gender, age, etc.), in practice it is aggregated when the number of clusters depends on the amount of available data.

The model aggregation problem is not considered in this work, and instead we focus on latent contexts for the rest of the paper. Additionally, we assume the initial state distribution and rewards are context independent, maintaining the hardness of the problem while greatly simplifying the writing. Finally, we adopt the common $[0,1]$-bounded reward assumption.

\subsection{General Setup} 

We define the general setup as follows: The context space consists of $K$ possible contexts. The time axis is divided into $H$ episodes, denoted by $e_1,\ldots,e_H$. In the beginning of each episode, the environment chooses a context $c \in \CSp$ (in a random, adversarial or any other fashion). Afterwards, an initial state is randomly chosen according to an initial state distribution $\pi^c_0$. A trajectory of length $T$ is generated where $T$ is a \emph{stopping time} \cite{meyn2009markov}. Then, for the chosen MDP an interaction as described in Definition \ref{def:MDP} is applied until the end of the trajectory.


\section{Problem Definition and Solution} \label{sec:ProblemDefinition}

We assume a small finite $\CSp$, and that $T$ is bounded almost surely, denoting this setup as \emph{finite sources episodic CMDP}. The goal is maximizing over the cumulative rewards from all trajectories by the $H$'th trajectory, for increasing $H$. Therefore, we measured performance with respect to $H$. Ideally, a good policy should optimize the trade-off between exploration and exploitation of the current chain. However, unlike the standard RL setup, the exploration in this case should consider not only the model's parameters, but also the hidden context. 

We measure our performance with the notion of \emph{regret}: the difference between the cumulative reward and the cumulative reward obtained by an agent satisfying some optimality property. For example, in infinite horizon RL the cumulative discounted reward is compared against an agent with knowledge of the true model who can therefore start from the optimal policy \cite{auer2009near}; the faster the regret bound converges to $0$ with $T$ the better. 

Similarly, we compare ourselves to the all knowing agent applying the optimal policy for the correct context at each trajectory. In our setup, since $T$ is bounded the regret is evaluated mainly with respect to the number of trajectories $H$. Notice though, that in each new trajectory some loss is guaranteed until the correct context is identified. Therefore, the regret will always be linear in $H$. A different optimal agent, when there is some prior distribution over contexts, can be chosen to perform the solution of the resulting POMDP, but there may be other appropriate choices. The problem of redefining the regret to obtain more meaningful bounds was left for future research.

\begin{Definition}
	For the problem of finite sources episodic CMDP we define the \textbf{regret} over $H$ trajectories to be:
	\begin{equation}
	Regret = \sum_{h=1}^H J^*_h - \sum_{h=1}^H \sum_{t=1}^{T_h} r_{h,t} \quad\quad,
	\end{equation}
	where $J^*_h$ is the optimal value function in $T_h$ steps for the context chosen in the $h$'th trajectory, and $r_{h,t}$ is the reward obtained by the agent in the $h$'th trajectory at the $t$'th step.
\end{Definition}

In order to solve the problem of regret minimization we introduce the CECE general framework (Cluster-Explore-Classify-Exploit) that partitions the trajectories to mini-batches. In the beginning of each mini-batch, all previously seen trajectories are used to form $K$ distinct models through Algorithm 1 (Cluster). Then, for each new trajectory in the current mini-batch the agent generates a partial trajectory using Algorithm 2 (Explore). The partial trajectory is then classified to a context by Algorithm 3 (Classify). Finally, Algorithm 4 sets the policy for the remainder of the trajectory (Exploit). In summary:

\begin{mdframed}
\begin{enumerate}[leftmargin=36pt, labelwidth=!, label=\bfseries Alg. \arabic*:]
	\item \textbf{C}luster observed trajectories to $K$ models.
	\item \textbf{E}xplore the context.
	\item \textbf{C}lassify partial trajectory to model.
	\item \textbf{E}xploit the identified model.
\end{enumerate}
\end{mdframed}

The following assumptions and theorem guarantee CECE's performance:
\begin{Definition}
	1. Let: 
	\begin{equation}
	\begin{split}
	M_1 &= (\SSp, \ASp, p_1(y|x,a), r(x), \pi_0), 
	\\
	M_2 &= (\SSp, \ASp, p_2(y|x,a), r(x), \pi_0)
	\end{split}
	\end{equation}
	be two MDPs with the same state space, action space, rewards and initial state distribution. We define $M_2$ to be an $\epsilon$\textbf{-approximated model} of $M_1$ if for every state-action pair $(s, a) \in \SSp \times \ASp$:
	\begin{equation}
		\| {\Pr}_1 (\cdot | s, a) - {\Pr}_2 (\cdot | s, a)   \|_1 \leq \epsilon.
	\end{equation}
	
	2. Let: 
	\begin{equation}
	\begin{split}
	X_1 &= (\CSp_1, \SSp, \ASp, \mathcal{M}_1(c)), 
	\\
	X_2 &= (\CSp_2, \SSp, \ASp, \mathcal{M}_2(c))
	\end{split}
	\end{equation} 
	be two CMDPs with the same state and action space satisfying $|\CSp_1| = |\CSp_2|$. We define $X_2$ to be an $\epsilon$\textbf{-approximated CMDP} of $X_1$ if there exists a matching between the contexts $f:\CSp_1 \leftrightarrow \CSp_2$ such that for every $c \in \CSp_1$ we have that $\mathcal{M}_2(c)$ is an $\epsilon$-approximated model of $\mathcal{M}_1(f(c))$. 
\end{Definition}

\begin{Assumption} \label{Ass:Clustering}
	Let $H_0$ be some constant number of trajectories. For every $H > H_0$ there exists $\delta_1(H), \epsilon(H) > 0$, such that after applying Algorithm 1 on $H$ trajectories, with probability at least $1-\delta_1(H)$ the estimated $K$-models form an $\epsilon(H)$-approximated CMDP of the true CMDP.
\end{Assumption}
Assumption \ref{Ass:Clustering} guarantees that having enough trajectories will drive Algorithm 1 to output an approximated model for each context. It envelopes a hidden assumption that all contexts were observed enough times. Since there is some probability of error $\delta_1$, the clustering procedure must be repeated when more trajectories are presented to ensure diminishing regret; that is the reason a mini-batch scheme is applied.

\begin{Assumption} \label{Ass:ExplorationClassification}
	For every $\epsilon>0$, there exists $\delta_2(\epsilon)$ such that given an $\epsilon$-approximated CMDP, after applying Algorithms 2 and 3 the correct context is identified with probability at least $1-\delta_2(\epsilon)$. In addition, the number of steps taken is a stopping time denoted by $T_{EC}$. 
\end{Assumption}
This assumption assures us each trajectory will be classified correctly with high probability, which will guarantee good performance for exploitation in the next step. Moreover, $T_{EC}$ represents the number of samples needed to differentiate between the models. 

\begin{Assumption} \label{Ass:Exploitation}
	Given an $\epsilon$-approximated model, Algorithms 4 obtains $\textit{Regret} \leq \zeta(\epsilon) $. 
\end{Assumption}
Assumption \ref{Ass:Exploitation} establishes the regret provided by Algorithm 4 when the models are well-approximated.


\begin{Theorem} \label{Thm:CECE}
	Let $H_i$ be the number of trajectories in the $i$'th mini-batch. Then if Assumptions \ref{Ass:Clustering}, \ref{Ass:ExplorationClassification}, \ref{Ass:Exploitation} hold, CECE achieves in the $L$'th mini-batch:
	\begin{equation} 
	\begin{split}
		& Regret \leq \quad (1-\delta_1) H_L ( \delta_2 \mathbb{E} T + (1-\delta_2)(\zeta+\mathbb{E} T_{EC}) )
		\\
		&  \quad\quad\quad\quad\quad\quad\quad + \delta_1 H_L \mathbb{E} T 
	\end{split}
	\end{equation}
	where $\delta_1=\delta_1(\overline{H}), \epsilon=\epsilon_1(\overline{H}), \delta_2 = \delta_2(\epsilon), \zeta=\zeta(\epsilon)$ and $\overline{H}=\sum_{i=1}^{L-1} H_i$. 
\end{Theorem}

The proof is a straightforward combination  of the given assumptions. 

\subsection{Discussion}
Notice that in order for Assumption $1$ to hold with a meaningful $\epsilon$, when $H_1$ is set each model must be observed sufficiently. This fact should be added as an additional assumption depending on the specific realization of Algorithm 1. Supposedly the subsequent $H_i$'s can be chosen arbitrarily small, utilizing information from new trajectories as soon as it is available. Yet, Algorithm 1 may be computationally expensive, making larger $H_i$'s preferable in practice. Another possible approach to this trade-off is to apply on-line clustering \cite{ailon2009streaming}. 

In essence, Algorithm 1 is a form of Multiple Model Learning (MML) algorithm \cite{vainsencher2013learning} -- each trajectory is a sample from an unknown model (context) and the goal is learning all models simultaneously. It could also be reduced to the clustering problem, where each trajectory is represented as an $S \times S \times A$ vector of its empirical transition matrix. Indeed, some information is lost in this process: the number of samples from each $(s, a)$ pair in the trajectory is ignored despite its effect on the variance around the sampled distribution. So, ideally each trajectory should be reduced to a point with varying variance across dimensions, which gets smaller for longer trajectories. 

Subsequently, one may question whether $\epsilon(H)$ can converge to $0$ for infinitely many trajectories. In our setup, as $T$ grows the trajectories are more distinct, but $T$ is bounded almost surely. So even for large $T$'s, there would be at least some constant portion of the trajectories acting as outliers of the model they originated from, possibly tainting the clusters. One way to solve this issue is through an outlier robust clustering (for example K-median ; \citealt{har2004coresets}).

Next, consider the effect of the trajectories length $T$ on the hardness of the problem. When $T$ is very large, it is much more important to recognize the correct model. Since Algorithm $4$ (exploitation) is applied for a longer duration, it could include an exploratory part to obtain a better model while running the trajectory, in addition to shielding against wrongful classification.

The other extreme case is when $T$ is too small to determine the correct model with high probability. Assuming the models can still be approximated, one reasonable solution would be to try and optimize the worst case performance over all models. This approach is closely related to the problem of Robust MDPs \cite{nilim2005robust} - a formulation of MDPs with uncertainty in the transitions and rewards. When the uncertainty set is rectangular an efficient solution exists. However, in our case it is singular - setting one transiton probability is the same as setting the context along with its related transition matrix; thus the problem is intractable \cite{wiesemann2013robust}. 

When all trajectories are short, it might be impossible to provide an approximation of the true models. Consider for example the extreme case where only one transition is given - unless there is a stationary distribution over contexts the models cannot be learned nor optimized. Subsequently, varied $T$ lengths pose another question: how confident are we in the clustering of each trajectory? Embedding short trajectories might inject more noise to the clustering process than improve it, so some selection is needed to insure proper modeling. This question may relate to the notion of clusters separability \cite{ostrovsky2006effectiveness} - short trajectories can lead to non-separable models that cannot be learned through clustering.

A rather simple realization of Algorithm 2 (exploration) is to apply a fixed policy until some condition is fulfilled. One may consider what is the policy which will achieve this condition with as few steps as possible (since the regret is linear in the number of exploration steps). 

For instance, if there are only two models a logical approach would be to choose actions maximizing the distinction between the models. However, this is non-optimal as actions have future consequences - a distinctive action for one state could lead the agent to an area of the state space which is very similar between the models. 

A follow-up idea is using the original state and action space, and reshaping the rewards to award actions for distinguishing between the models. However,  this solution is still problematic since the underlying transition probabilities are unknown and could be these of either of the possible models. Hence, finding a good exploration policy is an open question we hypothesize to be as difficult as solving a singular Robust MDP.

Finally, consider the effect of increasingly more possible contexts. These increase both the size of the initial $H_1$ required for clustering, and the number of samples needed for model identification $T_{EC}$. The case of infinitely many models requires some changes in the algorithm, as discussed in the end of this section.  

\subsection{A Specific Instance}

We an example for an instance of CECE and substitute in Assumptions \ref{Ass:Clustering}, \ref{Ass:ExplorationClassification}, \ref{Ass:Exploitation}. For simplicity, we assume the trajectory length is a constant $T$ for the remainder of the analysis. The proposed realization was chosen to be trivial to allow simple analysis; It is only a demonstration of the trade-offs in CMDPs and CECE's modularity.

Algorithm 1 is the following scheme:
\begin{enumerate}
	\item For each trajectory $h$, and state action pair $(s, a)$, estimate the transition probability $\widehat{\Pr}_h(\cdot | s, a)$ by its empirical distribution. 
	\item Go over all possible partitions of trajectories to $K$ sets $\{C_k\}_{k=1}^K$, and minimize over the following score:
	\begin{equation}
	\sum_{k=1}^K \sum_{h \in C_k} \max_{s,a} \| \widehat{\Pr}_h(\cdot | s, a) - \widehat{\Pr}_k(\cdot | s, a) \|_1,
	\end{equation}
	where $\widehat{\Pr}_k(\cdot | s, a)$ is the estimated transition probability for all trajectories in the cluster.
\end{enumerate}

This scheme is highly inefficient as it performs an exhaustive search for the best partition. However, as a preliminary result all we require is for it to accommodate Assumption $1$. There are other polynomial time clustering algorithms with guarantees (\citealt{ostrovsky2006effectiveness, arthur2007k} for instance), but their bounds and assumptions would have to be adjusted to our case. 

In Algorithm 2, the uniform policy over actions is applied for a constant number of steps $T_{EC}$. As mentioned above, this procedure could be improved. For once, the total number of steps could be decided on-line according to the confidence. Moreover, there might be other exploration policies that could produce faster identification of the true model, or even combine exploitation in the strategy to generate overall smaller regret.
	
The proposed Algorithm 3 chooses the model obtaining the smallest $L_1$ distance between the set of models and the empirical transition matrix from the partial trajectory. Other possible methods include maximum likelihood, weighted $L_1$ or $L_2$ distance, and methods taking into account the cost of choosing a wrong model.

Lastly, Algorithm 4 was chosen naively to apply the exploitation policy with regards to the estimated model. A more sophisticated approach would be to consider an RL algorithm whose regret with respect to $T$ goes to $0$. Since in our scenario $T$ is constant, the suggested solution is satisfactory. 

We can now quote the necessary assumptions and resulting Corollary:
\begin{Assumption}\label{Ass:Regularity}
Let $\alpha, \beta \in (0,1)$.
\begin{enumerate}	
	\item By the $H$'th trajectory, each model was sampled at least $\beta H$ times.
	\item For some $D$, for every two contexts $c_1, c_2$ and $s, a$: $\| {\Pr}_{c_1} (\cdot | s, a) - {\Pr}_{c_2} (\cdot | s, a)   \| \geq D $.
	\item In every trajectory, each state-action pair is visited at least $\alpha T$ times, and $T$ is large enough: $T \in O(\frac{S}{\alpha D^2} \log (\frac{D}{K S A}))$.
\end{enumerate}
\end{Assumption}
The first part guarantees each model is sampled enough times for the classification to converge. The second part provides a constant difference between the models, such that with enough data the estimated models will be separable. The last part of the assumption is needed to make sure there are enough samples in each trajectory to learn the model. It can be guaranteed by requiring $T_{EC}$ to be long enough, assuming that the induced MDP is ergodic under the uniform policy. 

\begin{Lemma} \label{Lem:Realization}
	If Assumption \ref{Ass:Regularity} holds, the described realization of Algorithms 1-4 satisfy Assumptions 1-3 with:
	\begin{equation}
	\begin{split}
		\epsilon(H) & \in O(K S A  e^{S- \alpha T D^2}),
		\\
		\delta_1(H) & \in O( K S A  e^{S- \alpha T \beta H D^2}),
		\\
		\delta_2(\epsilon) & \in O(K e^{S-T_{EC} (\frac{D}{2} - \epsilon)^2 }), \quad D > 2\epsilon
		\\
		\zeta(\epsilon) & \in O(S^2 T^2 \epsilon).
	\end{split}
	\end{equation}
\end{Lemma}

	The full proof is available in Section \ref{prf:Realization} of the supplementary material.
	
\begin{Corollary}
	If Assumption \ref{Ass:Regularity} holds, the described realization of Algorithms 1-4 achieves in the $L$'th mini-batch:
	\begin{equation} 
	\begin{split}
		Regret \leq & O( H_L T K e^{S-T_{EC} D^2 / 4 })
		\\
		 & + O( H_L T^2 K S^3 A e^{S- \alpha T D^2} + H_L T_{EC}) 
		\\
		 & + O( H_L T K S A e^{S- \alpha T \beta \overline{H} D^2})  ,
	\end{split}
	\end{equation}
	where $\overline{H}=\sum_{i=1}^{L-1} H_i$.
\end{Corollary}
Notice that each summand relates to a different error:
\begin{enumerate}
	\item The first summand corresponds to trajectory misclassification. It can point us to proper choice of $T_{EC}$: scaled with $S$ and the distance between models.
	\item The second summand corresponds to the context and model uncertainty. Large $T$ and $\alpha$ are required to estimate each model well enough.
	\item The third summand corresponds to trajectories misclustering. It is the only error which diminishes with $H$, as the exponential multiplicative converges to $0$.
\end{enumerate}


%
%
%
%

\subsection{Extensions}

There are other interesting extensions to the previous setup exhibiting different trade-offs. For once, consider the more complicated scenario when there is an \emph{infinite or unknown number of models}. CECE's mini-batch solution can be adjusted to this case by adding a probability to reject all models in Algorithm 3, but the clustering step will be much harder to evaluate in this case. Consequently, regret analysis requires a more precise setup, for example bounded ratio between the number of contexts and trajectories, or some distribution over contexts.

A more natural setup in web advertising applications is the \emph{concurrent RL} setup \cite{silver2013concurrent}. Assume the agent interacts with multiple infinite horizon trajectories, where each time step one trajectory (which may be new) requires an action. In the CMDP setup, each trajectory originates from a different latent context. The performance in this case should take into account both the length and number of trajectories.

A rather naive solution would be to employ some RL algorithm (for example, Q-learning; \citealt{watkins1992q}) in every trajectory, regardless of the other trajectories. This approach ignores information on the model obtained from other trajectories sharing the same context. Thus, if there are many short trajectories it could produce high regret.

A different solution is applying some variation of CECE's scheme - in each time step in a trajectory: first cluster (Algorithm 1), and then either (a) choose an option which explores the context (Algorithm 2), or (b) classify the partial trajectory (Algorithm 3) and choose an action exploiting the context (Algorithm 4). Even though the trajectory length is unbounded, as long as more model samples are obtained from other trajectories the error in the exploitation phase decreases. Actual regret bounds for both approaches depend on the parameters and assumptions of the specified problem. When there are few long trajectories the first independent RL approach would prevail (with regret of $O(H\sqrt{T})$; \citealt{auer2009near}) , while many shorter trajectories are better dealt with a CECE variant (with regret of $O(H T_{EC} + \sqrt{THK})$ for equal probability contexts).

\section{Experiments}\label{sec:experiments}

In this section we discuss the trade-offs that exist in the CMDPs settings. In the first experiment we test only the clustering part in CECE. We consider a CMDP with $K = 5$ equal probability contexts, $| \ASp | = 2$ actions and $| \SSp | = 100$ states where the transition matrix for each context was drawn from a uniform distribution. We generate $H$ trajectories of a constant length $T$ sampling actions uniformly. For the purpose of scoring the clusters we calculate the entropy of each distribution over clusters for each correct context, and average the results according to the number of samples from that context. Thus, when the trajectories are perfectly clustered, for each context the entropy will be $0$ and so will be the average. The worst possible score $\log (K)$ results from independent clusters and contexts. The clustering algorithm we used in this case was $K$-means \cite{duda2012pattern} on the vectorized empirical transition matrices, the results were averaged over $100$ trials and were added error bars of one standard deviation.

We examined the following: (1) How long should trajectories be to obtain favorable clustering? (2) How the quality of the clustering depends on the number of episodes, for various trajectories lengths? In the first part of the experiment (top plot in Figure \ref{fig:C-MC}) we generate $H=100$ trajectories and present the score as a function of the trajectories length $T$. In the second part of the experiment (bottom plot in Figure \ref{fig:C-MC}), we generate trajectories of varying lengths $T=2000, 5000, 8000$ and measure the score as a function of the number of episodes $H$.


\begin{figure}[h]
	\caption{Experiment 1 \label{fig:C-MC}}
	\centering
	\includegraphics[width=0.5\textwidth]{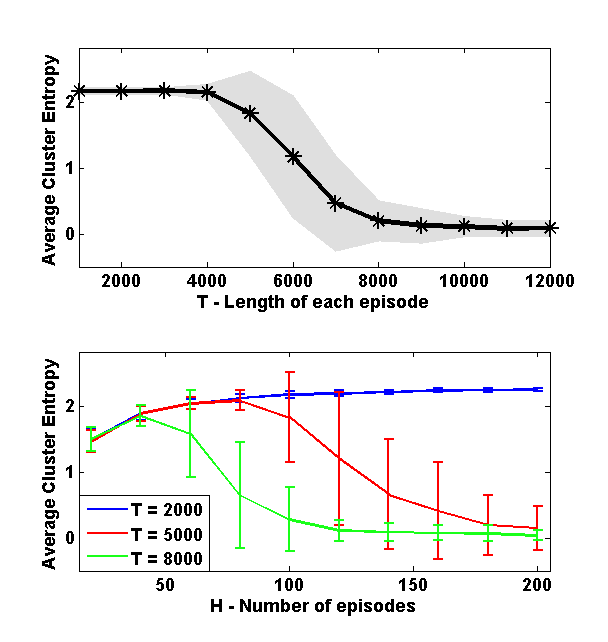}
\end{figure}
We draw the following conclusions: (1) There is a phase transition in the clustering performance with respect to $T$: below a certain threshold (here $T=4000$) the clustering utterly fails, followed by a short adjustment period, where finally (here at $T=8000$) the clustering succeeds almost certainly. (2) If the trajectories are too short, the clustering will fail even when increasing the number of episodes. (3) If the trajectories are sufficiently long, additional episodes improve the clustering quality (as implied by Lemma \ref{Lem:Realization}).

Next, we experimented with the full CECE algorithm. We simulated a CMDP with $|\SSp|=100$ states, $|\ASp|=4$ actions and $K=20$ contexts of equal probability. Each trial consists of $H=100$ episodes of length $T=2000$. The results were averaged over $20$ experiments. The parameter $T_{EC}$ sets the portion of the trajectory time steps dedicated to identify the model, and was taken to be $\eta \cdot T$, $\eta = 0.3$. The learning policy employed by Algorithm 2 was taken to be uniform over all actions. The exploitation algorithm used is Q-learning \cite{bertsekas1995neuro}. 

We performed four experiments where in each of the experiments all the parameters excluding one were fixed. The average reward throughout the experiment is measured. The results are presented in Figure \ref{fig:CMDP}. On the top-left and bottom-right plots we can see how CECE behaves as the number of episodes and trajectory length increase. As more data are available, the average reward increases since the clustering phase performs better and the models are better learned. Similarly, the average reward decreases as more models are introduced (top-right plot) since it is harder to cluster and learn each model. Notice that for constant proportion $\frac{T_{EC}}{T}$ there will always be a difference between the optimal and the achieved value due to the identification phase. 

\begin{figure}[h]
	\caption{Experiment 2 \label{fig:CMDP}}
	\centering
	\includegraphics[width=0.48\textwidth]{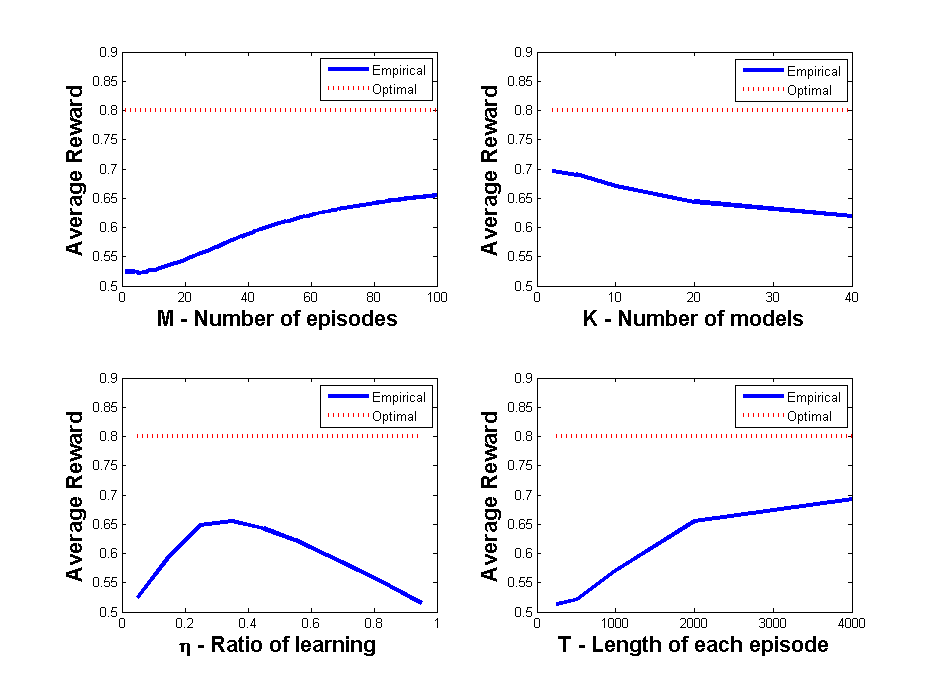}
\end{figure}

An interesting result is presented in the bottom-left plot. The parameter $\eta=\frac{T_{EC}}{T}$ describing the portion of samples taken to identify the correct model. The resulting plot represents the exploration-exploitation trade-off for our suggested model: How many samples are used to identify the correct model against how many of them are used to optimize the C-MDP.

\section{Conclusions and Future Work} \label{sec:Conclusions}

In this work we presented a new framework for modeling multiple Markovian sources with sequential decision making. While our models can be encompassed in existing models (e.g., POMDPs; \citealt{aberdeen2003revised}) the proposed setup offers much flexibility in modeling both observable and latent static context while maintaining computational tractability. We demonstrated that under certain conditions one can overcome two fundamental problems: (1) learning the model parameters, and (2) optimizing on-line the action within an RL framework. We suggested and analyzed basic algorithms when the number of contexts is finite.

This paper is but a first step in developing the contextual MDPs framework. Since CECE is a modular solution its performance can be improved by independent upgrades to its building blocks, such as: 
\begin{enumerate}
	\item The clustering techniques we used are somewhat inefficient and does not consider the confidence of each trajectory. 
	\item Data and models dependent learning policies could possibly classify the trajectory in less steps.
	\item Reward oriented context classification can lead to improved overall regret.
	\item Incorporating context exploration in the exploitation phase hedges against miss-classification.
\end{enumerate}

There are other schemes to solve CMDPs. A rather similar approach is combining the Exploration-Classification-Exploitation steps to form a belief over models and solve accordingly (like MMRL; \citealt{doya2002multiple}). Another reasonable approach when there is some distribution over contexts is to model the problem as a POMDP \cite{aberdeen2003revised}, and then learning and optimizing it. Finally, it is possible to view the optimization problem as a robust MDP (\citealt{nilim2005robust}; where uncertainty is on which model the data come from). While solving the resulting Robust MDPs directly is hard computationally, a rectangular relaxation can be possibly used to provide an approximated result; one future direction is to investigate this approximation. 

The concurrent RL setup \cite{silver2013concurrent}, as well as the case of many or even infinitely many contexts are of practical importance. We have presented rough ideas on how to pursue these, but the exact theoretical setup requires a more precise definition (what guarantees could be made, what assumptions must hold and so on). 

The issues of computational efficiency and sample complexity are important and were not tackled in this paper. Despite the availability of big data in many appealing venues,  the state, action and context spaces may scale accordingly. Hence, an interesting theoretical and practical concern is the error and regret rates for finite sample size; finding these requires a more subtle analysis and is left for future work. 

Subsequently, for very large state or action spaces, straightforward implementation of the model-based approach will fail as the number of samples required to learn the model grows accordingly. Solving this problem within the CMDP framework may introduce some intriguing connections. For example, if the linear function approximation technique is used \cite{sutton1998introduction}, the problem of clustering same-policy trajectories corresponds to the subspace clustering problem \cite{vidal2010tutorial}. 

In conclusion, from an algorithmic and analytic points of view the theoretical trade-off between learning, exploration, optimization, and control of CMDPs is still very much an open question.

\bibliography{ContextualBib}{}
\bibliographystyle{icml2014}

\newpage
\appendix
\onecolumn

%
%
\section{List of Notations} \label{app:notations}
\begin{tabular}{|l|l|}
	\hline
	Notation & Meaning \\ \hline
	$\SSp$ & State space or number of states \\ \hline		
	$\ASp$ & Action space or number of actions \\ \hline	
	$T$ & Time horizon \\ \hline
	$t$ & Time index $t=0..T$ \\ \hline
	$H$ & Number of trajectories in batch data \\ \hline
	$H_L$ & Number of trajectories in the $L$'th mini-batch \\ \hline
	$\CSp$ & Number of possible contexts \\ \hline
	$J^\mu_M$ & Value of policy $\mu$ in model $M$ \\ \hline
	$D$ & Minimal inf-distance between two distinct models. \\ \hline
\end{tabular}

%
%

\section{Useful Lemmas}
The following Lemmas are used in the proofs:

\begin{Lemma}{\cite{weissman2003inequalities}}\label{Lem:Weissman}
	Let $P$ be a probability distribution on the set $\SSp={1,..,S}$. Let $\mathbb{X}^m = X_1,X_2,...,X_m$ be independent identically distributed random variables distributed according to $P$. Then for all $\epsilon > 0$,
	\begin{equation}
	\Pr (\| P - \hat{P}_{\mathbb{X}^m} \|_1 \geq \epsilon ) \leq e^{S-m\epsilon^2 / 2}
	\end{equation}
\end{Lemma}

\begin{Lemma}{\cite{kearns2002near}}
	Let $M$ be an MDP over $S$ states, and $\hat{M}$ be an $O(\epsilon)$-approximation of $M$. Then for any policy $\mu$: 
	\begin{equation}
		| J^\mu_M - J^\mu_{\hat{M}} | \leq S^2 T^2 \epsilon ,
	\end{equation}
	and consequently for the optimal policy in each MDP correspondingly:
	\begin{equation}
		| J^*_M - J^*_{\hat{M}} | \leq 3 S^2 T^2 \epsilon , 
	\end{equation}
\end{Lemma}


%
%

\section{Proof of Lemma \ref{Lem:Realization}} \label{prf:Realization}

\setcounter{Lemma}{0}
\begin{Lemma} 
	If Assumption \ref{Ass:Regularity} holds, the described realization of Algorithms 1-4 satisfy Assumptions 1-3 with:
	\begin{equation}
	\begin{split}
	& \epsilon(H) \in O(K S A  e^{S- \alpha T D^2}),
	\\
	& \delta_1(H) \in O( K S A  e^{S- \alpha T \beta H D^2}),
	\\
	& \delta_2(\epsilon) \in O(K e^{S-T_{EC} (\frac{D}{2} - \epsilon)^2 }), \quad D > 2\epsilon
	\\
	& \zeta(\epsilon) \in O(S^2 T^2 \epsilon).
	\end{split}
	\end{equation}
\end{Lemma}

\begin{proof}
	We show each Assumption holds, starting with Assumption \ref{Ass:Clustering}.
	
	For two transition functions $P_1, P_2$ of size $S\times S \times A$ denote:
	\begin{equation}
		\| P_1 - P_2 \| \triangleq \max_{s, a} \| P_1(\cdot | s, a) - P_2(\cdot | s, a) \|	.	
	\end{equation}
	We denote by $\widehat{\Pr}_h, \widehat{\Pr}_c$ the estimated transition matrices from trajectory $h$ and cluster $c$ correspondingly. In addition, $C^*$ is the true clustering of each trajectory, and $C^{opt}$ is the clustering found by the algorithm. 
	
	Since there are at least $\alpha T$ samples from each state-action pair, according to Lemma \ref{Lem:Weissman} and the union bound, we obtain that:
	\begin{equation}
	\Pr(\| \widehat{\Pr}_h (\cdot|s, a) - {\Pr}_{C^*(h)} (\cdot | s, a) \| \leq \epsilon ) \geq 1 - S A  e^{S- \alpha T \epsilon^2 / 2}.
	\end{equation}	
	Since there are at least $\beta H$ trajectories from each model, we also obtain that:
	\begin{equation}
	\Pr(\| \widehat{\Pr}_{C^*(h)} (\cdot|s, a) - {\Pr}_{C^*(h)} (\cdot | s, a) \| \leq \epsilon ) \geq 1 - S A  e^{S- \alpha T \beta H \epsilon^2 / 2},
	\end{equation}
	and therefore:
	\begin{equation}
	\Pr(\sum_{h=1}^H \| \widehat{\Pr}_{C^*(h)} (\cdot|s, a) - {\Pr}_{C^*(h)} (\cdot | s, a) \| \leq H\epsilon ) \geq 1 - K S A  e^{S- \alpha T \beta H \epsilon^2 / 2},
	\end{equation}
	
	Now we obtain the following:
	\begin{equation}
	\begin{split}
		\sum_{h=1}^H \| P_{C^* (h)} - \hat{P}_{C^{opt} (h)} \| 
		& \leq \sum_{h=1}^H \|  P_h - P_{C^* (h)}  \| + \sum_{h=1}^H \|  P_h - \hat{P}_{C^{opt} (h)}  \|, \quad \textit{Triangle inequality}
		\\
		& \leq \sum_{h=1}^H \|  P_h - P_{C^* (h)}  \| + \sum_{h=1}^H \|  P_h - \hat{P}_{C^* (h)}  \|, \quad \textit{By Algorithm definition}
		\\
		& \leq 2\sum_{h=1}^H \|  P_h - P_{C^* (h)}  \| + \sum_{h=1}^H \|  P_{C^*(h)} - \hat{P}_{C^* (h)}  \|, \quad \textit{Triangle inequality (second term)}.
	\end{split}
	\end{equation}
	When $H$ is large, we can approximate $2\sum_{h=1}^H \|  P_h - P_{C^* (h)} \| \in O( H(1-\delta) \epsilon + H \delta ) = O( H\epsilon + H\delta )$ for $\delta= S A  e^{S- \alpha T \epsilon^2 / 2}$ since each summand is bounded by $\epsilon$ with that probability, and when it is unbounded the maximal value of $L_1$ distance between two distributions is a constant $2$. Therefore:
	\begin{equation}
		\frac{1}{H}\sum_{h=1}^H \| P_{C^* (h)} - \hat{P}_{C^{opt} (h)} \| \in O( \epsilon + \delta )
	\end{equation}
	with probability at least $1 - K S A  e^{S- \alpha T \beta H \epsilon^2 / 2}$, for $\delta= S A  e^{S- \alpha T \epsilon^2 / 2}$.
	
	Since the average is of that order, there must exist a matching between the true clusters and optimal clusters satisfying:
	\begin{equation}
	 \max_{c \in C*} \| P_c - \hat{P}_{c^{opt} } \| \in O( \epsilon + \delta )
	\end{equation}
	If the distance between every two true clusters is $D > O(\epsilon + \delta )$, the agreement between matching clusters are on all trajectories in a reasonable radius, i.e. $O(1-\delta)$ of the trajectories. so the error in each model is of the order $O(K\delta)$:
	\begin{equation}
		 \| {\Pr}_c (\cdot | s, a) - \widehat{\Pr}_i (\cdot | s, a)   \| \leq K S A  e^{S- \alpha T \epsilon^2 / 2}.
	\end{equation}
	Now in order for $D > O(\epsilon + \delta )$ to hold, we can choose $\epsilon$ to be of order $D$, and then:
	\begin{equation}
		 K S A  e^{S- \alpha T \epsilon^2 / 2} \in O(D) \Rightarrow T \in O(\frac{S}{\alpha D^2} \log (\frac{D}{K S A})).
	\end{equation}	
	To summarize, for $T \in O(\frac{S}{\alpha D^2} \log (\frac{D}{K S A}))$ we obtain that with probability at least $1 - \delta(H)$, $\| {\Pr}_c (\cdot | s, a) - \widehat{\Pr}_i (\cdot | s, a)   \| \leq \epsilon(H)$, where:
	\begin{equation}
	\epsilon(H) \in O(K S A  e^{S- \alpha T D^2}), \quad\quad \delta(H) \in O( K S A  e^{S- \alpha T \beta H D^2}).
	\end{equation}
	
	Next, we show Assumption \ref{Ass:ExplorationClassification} holds. We bound the probability of misclassification by the following probability:
	\begin{equation}
		\Pr(\| \widehat{P}_h - \widehat{P}_{C(h)} \| \leq \frac{D}{2}, \| \widehat{P}_h - \widehat{P}_{c \neq C(h)} \| \geq \frac{D}{2}),
	\end{equation}
	as if this event occurs then the true model will be chosen. To bound this quantity, we use the union bound over the complement event, so we need to bound:
	\begin{equation}
	\Pr(\| \widehat{P}_h - {\widehat{P}}_{C(h)} \| \geq \frac{D}{2}), \quad\quad \Pr( \| \widehat{P}_h - {\widehat{P}}_{c \neq C(h)} \| \leq \frac{D}{2}).
	\end{equation}
	For the left term:
	\begin{equation}
	\begin{split}
	\Pr(\| \widehat{P}_h - {\widehat{P}}_{C(h)} \| \geq \frac{D}{2}) & \leq  \Pr(\| \widehat{P}_h - P_{C(h)} \|  + \| P_{C(h)} - {\widehat{P}}_{C(h)} \| \geq \frac{D}{2})
	\\
	& \leq  \Pr(\| \widehat{P}_h - P_{C(h)} \|  \geq \frac{D}{2} - \epsilon)
	\\
	& \leq  e^{S-T_{EC} (\frac{D}{2} - \epsilon)^2 / 2}, \quad\quad \textit{Lemma \ref{Lem:Weissman}} 
	\end{split}
	\end{equation}
	For the right term:
	\begin{equation}
	\begin{split}
	\Pr( \| \widehat{P}_h - {\widehat{P}}_{c \neq C(h)} \| \leq \frac{D}{2}) & \leq  \Pr(\| \widehat{P}_h - P_{C(h)} \|  - \| P_{c \neq C(h)} - P_{C(h)} \| - \| P_{c \neq C(h)} - {\widehat{P}}_{c \neq C(h)} \| \leq \frac{D}{2})
	\\
	& \leq  \Pr(\| \widehat{P}_h - P_{C(h)} \|  \leq \frac{D}{2} + \epsilon + D)
	\\
	& \leq  e^{S-T_{EC} (\frac{3D}{2} + \epsilon)^2 / 2}, \quad\quad \textit{Lemma \ref{Lem:Weissman}} 
	\end{split}
	\end{equation}	
	
	Now, using the union bound we obtain that the classification is correct with probability at least $1-\delta$, where
	\begin{equation}
	\delta = e^{S-T_{EC} (\frac{D}{2} - \epsilon)^2 / 2} + K e^{S-T_{EC} (\frac{3D}{2} + \epsilon)^2 / 2}
	\end{equation}
%
\end{proof}

\end{document}